\documentclass{article}

\usepackage[letterpaper,margin=1in]{geometry}

\usepackage[T1]{fontenc}
\usepackage{microtype}

\usepackage{titlesec}
\titleformat*{\paragraph}{\bfseries}

\usepackage{amsthm,amsmath,amssymb,amsfonts,amssymb,mathtools}
\usepackage{xcolor}
\usepackage{empheq}
\usepackage{dsfont}
\usepackage{mathrsfs}
\usepackage{graphicx}
\usepackage{enumitem}
\usepackage{subfig}

\usepackage{tikz}
\usetikzlibrary{patterns}
\usetikzlibrary{arrows,shapes,automata,backgrounds,petri,positioning}
\usetikzlibrary{shadows}
\usetikzlibrary{calc}
\usetikzlibrary{spy}
\usepackage{pgf,pgfplots}
\usetikzlibrary{angles, quotes}
\usepackage{pgfmath,pgffor}
\pgfplotsset{compat=1.17}
\usepackage{framed}
\usepackage{algorithm2e}
\usepackage[noend]{algpseudocode}
\SetAlCapFnt{\normalfont\small}\SetAlCapNameFnt{\unskip\itshape\small}

\definecolor[named]{ACMBlue}{cmyk}{1,0.1,0,0.1}
\definecolor[named]{ACMYellow}{cmyk}{0,0.16,1,0}
\definecolor[named]{ACMOrange}{cmyk}{0,0.42,1,0.01}
\definecolor[named]{ACMRed}{cmyk}{0,0.90,0.86,0}
\definecolor[named]{ACMLightBlue}{cmyk}{0.49,0.01,0,0}
\definecolor[named]{ACMGreen}{cmyk}{0.20,0,1,0.19}
\definecolor[named]{ACMPurple}{cmyk}{0.55,1,0,0.15}
\definecolor[named]{ACMDarkBlue}{cmyk}{1,0.58,0,0.21}

\usepackage[colorlinks,citecolor=blue,linkcolor=magenta,bookmarks=true]{hyperref}
\usepackage[nameinlink]{cleveref}
\crefname{ineq}{Inequality}{Inequality}
\creflabelformat{ineq}{#2{\upshape(#1)}#3}
\crefname{sub}{Subsection}{Subsection}
\creflabelformat{Subsection}{#2{\upshape(#1)}#3}
\crefname{sdp}{SDP}{SDP}
\creflabelformat{sdp}{#2{\upshape(#1)}#3}
\crefname{lp}{LP}{LP}
\creflabelformat{lp}{#2{\upshape(#1)}#3}

\newtheorem{theorem}{Theorem}[section]

\newtheorem{lemma}[theorem]{Lemma}
\newtheorem{informal theorem}[theorem]{Theorem (informal statement)}

\newtheorem{corollary}[theorem]{Corollary}
\newtheorem{claim}[theorem]{Claim}
\newtheorem{fact}[theorem]{Fact}

\newtheorem{remark}[theorem]{Remark}

\newtheorem{definition}[theorem]{Definition}
\newcommand{\eqdef}{\coloneqq}

\renewcommand\vec[1]{\boldsymbol{#1}}
\DeclareMathOperator*{\pr}{\mathbf{Pr}}
\DeclareMathOperator*{\E}{\mathbf{E}}

\DeclareMathOperator*{\argmax}{argmax}

\newcommand{\bx}{\mathbf{x}}
\newcommand{\by}{\mathbf{y}}

\newcommand{\R}{\mathbb{R}}

\newcommand{\Z}{\mathbb{Z}}

\newcommand{\eps}{\epsilon}

\newcommand{\poly}{\mathrm{poly}}
\newcommand{\polylog}{\mathrm{polylog}}

\newcommand{\sgn}{\mathrm{sign}}
\newcommand{\sign}{\mathrm{sign}}

\newcommand{\D}{\mathcal{D}}

\newcommand{\Ind}{\mathds{1}}
\newcommand{\1}{\Ind}
\newcommand\matr[1]{\mathbf{#1}}

\newcommand{\wt}{\widetilde}

\newcommand{\wstar}{{\vec w}^{\ast}}

\newcommand{\x}{\vec x}
\newcommand{\z}{\vec z}
\newcommand{\w}{\vec w}

\newcommand{\tth}{^{(t)}}

\newcommand{\ocoreg}{\bar{R}}
\newcommand{\X}{\matr X_{1-2}^{(t)}}

\newcommand\Xt[1]{\matr X_{#1}^{(t)}}

\newcommand{\Deltapar}{\widetilde \Delta}
\newcommand\Gloss{ G}
\newcommand{\citep}{\cite}

\author{
Ilias Diakonikolas\thanks{Supported by NSF Medium Award CCF-2107079,
		NSF Award CCF-1652862 (CAREER), and
		a DARPA Learning with Less Labels (LwLL) grant.}\\
UW Madison\\
{\tt ilias@cs.wisc.edu}\\
\and
Vasilis Kontonis\thanks{This research was done at UW Madison. 
Supported in part by NSF Award CCF-2144298 (CAREER).}\\
UT Austin\\
{\tt vasilis@cs.utexas.edu } \\
\and
Christos Tzamos\thanks{This research was done at UW Madison. Supported by NSF Award CCF-2144298 (CAREER).}\\
UW Madison \& University of Athens\\
{\tt tzamos@wisc.edu}
\and
Nikos Zarifis\thanks{Supported by NSF Medium Award CCF-2107079, and a DARPA Learning with Less Labels (LwLL) grant.}\\
UW Madison\\
{\tt zarifis@wisc.edu}\\
}

\begin{document}

\title{Online Learning of Halfspaces with Massart Noise}

\maketitle
\abstract{ 
We study the task of online learning in the presence of Massart noise. Instead of assuming that the online adversary chooses an arbitrary sequence 
of labels, we assume that the context $\x$ is selected adversarially but 
the label $y$ presented to the learner disagrees with the ground-truth label 
of $\x$ with unknown probability at most $\eta$.  We study the fundamental 
class of $\gamma$-margin linear classifiers and 
present a computationally efficient algorithm that achieves mistake bound $\eta T + o(T)$.   
Our mistake bound is qualitatively tight for efficient algorithms: 
it is known that even in the offline setting achieving
classification error better than $\eta$ requires super-polynomial time in the SQ model.

We extend our online learning model to a $k$-arm contextual bandit setting where the rewards --- instead of satisfying commonly used realizability assumptions --- are consistent (in expectation) with some linear ranking function with weight vector $\vec w^\ast$.   
Given a list of contexts $\x_1,\ldots \x_k$, 
if $\vec w^*\cdot \x_i > \vec w^* \cdot \x_j$, the expected reward of action $i$
must be larger than that of $j$ by at least $\Delta$.  
We use our Massart online learner to design an efficient bandit algorithm 
that obtains expected reward at least 
$(1-1/k)~ \Delta T - o(T)$ bigger than choosing a random action at every round. 
\setcounter{page}{0}
\thispagestyle{empty}
\newpage

\section{Introduction}
\label{sec:motivation}

Online prediction has a rich history dating back to the works of \cite{robbins1951asymptotically, hannan1957approximation, blackwell1954controlled}. 
In the online scenario, the learner's objective is to tackle a prediction task by acquiring a hypothesis from a series of examples presented one at a time. The aim is to minimize the overall count of incorrect predictions, known as the mistake bound, considering the knowledge of correct answers to previously encountered examples \citep{Littlestone:88, Littlestone:89, Blum:90, LittlestoneWarmuth:94, MT:94}.  In the context of online linear classification, i.e., when the presented labels can be realized by a linear threshold function, the seminal perceptron algorithm \citep{Rosenblatt:58, Novikoff:62}, was the first online learning algorithm.

\paragraph{Realizable and Agnostic Online Learning} In \cite{Littlestone:88} the realizable online classification setting was defined, where the adversary is allowed to select an arbitrary datapoint $\x^{(t)}$ at every round but the label $y^{(t)}$ must be consistent with an underlying ground hypothesis $f$ from a class $\mathcal H$.
The number of mistakes in the realizable setting was shown to be characterized by the Littlestone
dimension $\mathrm{LD}(\mathcal H)$ of the class $\mathcal H$.
Similarly to the agnostic PAC learning of \cite{Haussler:92}, 
in \cite{ben2009agnostic}, motivated by the fact that often the observed labels are noisy,  
the setting of online learning with label noise was introduced. 
In the most extreme case of agnostic (adversarial) label noise where no assumptions
are placed on the labels, it was shown that the regret over $T$ rounds is $\wt{O}(\sqrt{T ~ \mathrm{LD}(\mathcal{H}) })$.   Even though the regret in the agnostic setting was shown to be sublinear in $T$, the corresponding computational task is far from being well-understood.  
In particular, even for simple classes such as linear classifiers efficient online 
algorithms with sublinear regret would imply efficient algorithms in the offline setting -- 
a well known computationally intractable problem \citep{GR:06}.  

\paragraph{Online Learning with Massart Noise} Investigating regimes beyond the above, worst-case, agnostic setting has been an important
area of research with the goal to get improved regret (and mistake) bounds but also to allow the design
of efficient algorithms.  Similar to the definition of Massart or Bounded noise \citep{Massart2006}
for offline (PAC) learning, in \cite{ben2009agnostic} a \emph{semi-random} online classification setting was introduced with a focus on improving over the agnostic regret bounds.  In this model, while the online adversary is allowed to pick arbitrary locations to present to the learner, 
the labels that they choose must be consistent with the ground-truth with probability strictly
larger than 50\%.
\begin{definition}[Online Learning with Massart Noise \citep{ben2009agnostic}]
\label{def:online-massart}
Fix a class of concepts $\mathcal C$ over $\R^d$ and a target concept 
$c^\ast \in \mathcal C$. At round $t = 1,\ldots, T$:
\begin{enumerate} 
\item The adversary selects a point 
$\x^{(t)} \in \R^d$ and generates a random label $y^{(t)}$ such that 
$\pr[y^{(t)} \neq c^\ast(\x^{(t)}) \mid \x^{(t)} ] \leq \eta$. 
\item The learner observes $\x^{(t)}$, predicts a label $\hat y^{(t)}$,
and suffers loss $\1\{\hat y^{(t)} \neq y^{(t)}\}$.
\end{enumerate}

The goal of the learner is to minimize the total mistakes defined as 
$M(T) = \E\Big[\sum_{t=1}^T \1\{ y^{(t)} \neq \hat y^{(t)} \} \Big]$. 
\end{definition}
\begin{remark}[Regret vs Mistake Bounds]
In \Cref{def:online-massart} often the regret over $T$ rounds is used, i.e.,
the expected difference between the mistakes done by the learner and the minimum
number of mistakes that any hypothesis $h$ in $\mathcal{H}$ could achieve $R(T, h) = 
M(T) - \E\Big[\sum_{t=1}^T \1\{y^{(t)} \neq h(\x^{(t)})\}\Big]$.  
\end{remark}

In \cite{ben2009agnostic}, an algorithm based on Littlestone's Standard Optimal Algorithm is given
that achieves a mistake bound $M(T) \leq \min_{h \in \mathcal H} \sum_{t=1}^T \1\{ h(\x^{(t)} \neq y^{(t)} + \mathrm{LD}(\mathcal H) \log T/(1 -2  \sqrt{\eta(1-\eta)})$.  Therefore, as long
as $\eta$ is bounded away from $1/2$, a strong separation between the agnostic regret (that grows roughly as $\Omega(\sqrt{T})$) and the Massart regret was established.  However, as discussed in \cite{ben2009agnostic}, this and other similar approaches for establishing regret bounds do not yield computationally efficient algorithms as they rely on computing combinatorial quantities that are at least as hard as computing the VC-dimension \citep{frances1998optimal}.  

\paragraph{Linear Classification}
To investigate the computational aspects of online classification we focus on the fundamental
class of linear classifiers or halfspaces, i.e., functions of the form $h(\x) = \sgn(\vec w \cdot \x)$ for some weight vector $\vec w\in \R^d$.
In the offline setting, a long line of recent works \cite{AwasthiBHU15,DGT19, DKTZ20b,CKMY20,DKKTZ21,DKKTZ21b} has successfully bypassed the computational hardness 
of worst-case agnostic learning and has provided efficient algorithms for linear classification in the presence of Massart noise.  However, no efficient algorithm that achieves any non-trivial mistake bound is known for the more challenging online Massart setting of \Cref{def:online-massart}.  In this work, we aim to answer the following fundamental question.
\begin{center}
\emph{
Are there {\em computationally efficient} algorithms for online linear classification with Massart noise?
}
\end{center}

\paragraph{Massart Bandits}
Going beyond full-information online learning, we consider the multi-armed bandit setting:
a multi-disciplinary research area that was initiated by \cite{thompson1933likelihood} 
and has received enormous attention in the past 30 years (see \cite{berry1985bandit, cesa2006prediction,slivkins2019introduction} and references therein).
In this work, we focus on the $k$-arm contextual bandit setting, where at every round $t$ the learner receives $k$ contexts $\x^{(t)}_1,\ldots \x^{(t)}_k$, chooses an action $\alpha = 1,\ldots,k$, and receives the reward/loss of their chosen action (the rewards of the other actions are not revealed). 
In this setting a common assumption known as realizability \cite{filippi2010parametric,abbasi2011improved, chu2011contextual,agarwal2012contextual,li2017provably, foster2018practical,foster2020beyond} prescribes that the expected rewards are parametric  
 functions of the contexts, e.g., $\E[r_i\mid \x^{(t)}] = \vec w \cdot \x^{(t)}_i$ for some unknown weight vector $\vec w$.
Using this structural assumption, those works typically reduce the bandit problem to an online
regression problem which can often be solved efficiently (for linear and generalized linear models).
An orthogonal direction \cite{langford2007epoch,dudik2011efficient,agarwal2014taming} makes minimal distributional assumptions and reduces the contextual bandit problem to agnostic (offline) classification.  Since agnostic classification is computationally intractable (for most non-trivial classes) such approaches do not provide end-to-end efficient algorithms and rely on heuristics to solve the underlying classification task.  Therefore, existing algorithms for contextual bandits either (i) make strong realizability assumptions and reduce the problem to (linear) regression or (ii) make minimal assumptions but face the computational intractability of
agnostic classification. Motivated by the online Massart classification setting of \cite{ben2009agnostic}, in this work, we propose a semi-random and semi-parametric Bandit 
model that lies between those two extreme settings and investigate the design of ``end-to-end'' 
efficient algorithms.  

Before we describe our semi-random noise model, we start with its ``noiseless'' version.
Instead of assuming that the expected rewards are given by a linear function of the contexts, we 
assume that they are \emph{ranked} according to their linear scores, i.e., if $\vec w^\ast \cdot \x_i > 
\vec w^\ast \cdot \x_j$ then the reward of action $i$ must be at least as that of action $j$.
\begin{definition}[Contexts and Linearly Sorted Rewards]
\label{def:linear-ranking}
Let $\mathcal{B}$ be the $d$-dimensional unit ball.
We define $\mathcal{X} = \{(\vec x_1, \ldots, \vec x_k)  : \vec x_1,\ldots, \vec x_k \in \mathcal{B}\}$ to 
be the context space.  For simplicity, we view each context as a $d\times k$ matrix $\matr X$.
Fix $M > 0$ and let $\vec w^\ast \in \R^d$ be some unit vector.
Given a context $(\vec x_1,\ldots, \vec x_k)$, we say that a reward vector $\vec r \in [0, M]^k$ is sorted according to the $\vec w^\ast$ if 
\[\text{for all $i, j$}: ~~
(\vec r_i - \vec r_j) (\vec w^\ast \cdot \vec x_i - \vec w^\ast \cdot \vec x_j) \geq 0\,.
\]
\end{definition}
Under the linear ranking setting of \Cref{def:linear-ranking} one can reduce the bandit problem
to a noiseless linear classification task and design efficient algorithms (based on perceptron or linear programming) that will eventually learn to select the action with the highest reward.  Our semi-random model extends the above definition where we require that rewards are only sorted in expectation with some margin $\Delta > 0$.

\begin{definition}[Monotone Reward Distributions]
\label{def:monotone-rewards}
Let $\mathcal{X}$ be the context space and $\mathcal{T} = \mathbb N$ be the set of rounds.
Fix some unit vector $\vec w^\ast \in \R^d$, $M >0$, and $\Delta >0$.
Define a class of distributions $\mathcal{D}$ indexed by rounds $t$ and contexts $\matr X$:
\(
\D = \{ D^{(t, \matr X)}: t \in \mathcal{T}, \matr X \in \mathcal{X} \} \,.
\)
We assume that each reward distribution $D^{(t, \matr X)}$ is supported on $[0,M]^k$.
We say that the class $\D$ has monotone rewards with respect to $\vec w^\ast$ 
with margin $\Delta$ if for all $t \in \mathcal{T}$,  
$\matr X \in \mathcal{X}$, $i, j \in [k]$ with $i \neq j$ it holds 
\[
 \E_{\vec r \sim D^{(t,\matr X)} }[\vec r_{i} - \vec r_{j} \mid 
 \vec w^\ast \cdot \vec x_i > \vec w^\ast \cdot \vec x_j  ]
 \geq \Delta 
 \,.
 \]
\end{definition}
\begin{remark}[Massart Online Classification as a 2-arm Bandit]
\label{rem:classification-to-bandit}
To obtain the online Massart setting we set the rewards $\vec r^{(t)} = 
((1+y^{(t)})/2, (1-y^{(t)})/2)$ and the contexts $\vec X^{(t)} = (\x^{(t)}, - \x^{(t)})$.
We observe that in that case when $y^{(t)}$ satisfies the Massart noise
condition of \Cref{def:online-massart}, it holds that conditional on $\vec w^*\cdot \x > 0$
it holds that $\E[(\vec r_1^{(t)} - \vec r_2^{(t)})] \geq 1-2 \eta$.  Therefore, \Cref{def:monotone-rewards} is satisfied with $\Delta = 1 -2 \eta$.
\end{remark}
We stress that the monotone reward distributions of \Cref{def:monotone-rewards} are a semi-parametric
model as we do not assume that the expected rewards have some parametric form rather than only require that they are sorted with respect to a linear sorting function.  A similar semi-random linear sorting model was 
used in \cite{fotakis2022linear} in the context of (offline) learning linear rankings with bounded noise 
under the Gaussian distribution.  We next define our contextual bandit model.

\begin{definition}[Contextual Bandits with Monotone Rewards (Massart Bandits)]
\label{def:massart-rewards}
Fix some unit vector $\vec w^\ast \in \R^d$, $\gamma, \Delta, M > 0$, 
and a class of monotone reward distributions $\D$ (see \Cref{def:monotone-rewards}).
At round $t$:
\begin{enumerate}

\item  The adversary picks context $\matr X^{(t)} \in \mathcal{X}$ and draws a random reward vector 
$\vec r\tth$ from $D^{(t, \matr X\tth)}$.
\item The learner observes the context $\matr X^{(t)}$, picks action $a^{(t)} \in [K]$ 
and receives reward $\vec r_{a^{(t)}}$.
\end{enumerate}
We also define the full-information setting the same way except from 
the last step where the learner receives reward $\vec r_{a^{(t)}}$ and 
observes the full reward vector $\vec r\tth$.
\end{definition}
\begin{remark}
{\em In what follows, we shall often simplify notation by writing 
$D^{(t)}$ instead of $D^{(t, \matr X)}$ for the reward distribution
at round $t$.}
\end{remark}

\subsection{Our Results}

Our first result answers our main question posed in \Cref{sec:motivation} and gives an efficient online classification algorithm that makes roughly $\eta T + o(T)$ mistakes in the Massart noise model
of \Cref{def:online-massart}.  Our algorithm only requires that the sequence of examples picked by the advesrary satisfies a standard $\gamma$-margin assumption.
Without the margin assumption it is known \cite{Littlestone:88} that, even in the noiseless setting, it is information theoretically impossible for the learner to do less than $T$ mistakes.  

\begin{theorem}\label{thm:2-arm}
Consider the Online Massart Learning setting of \Cref{def:online-massart}.
Additionally, assume that the examples picked by the adversary have at least
$\gamma$-margin with respect to some target halfpsace, i.e., 
for all $t=1,\ldots, T$, it holds that $\|\x^{(t)}\|_2 \leq 1$ and 
$|\vec w^\ast \cdot \x^{(t)}| \geq \gamma$, for some unit vector $\vec w^\ast$.
There exists an algorithm that does $M(T) = \eta T + O(T^{3/4}/\gamma)$ mistakes
and runs in $\poly(d)$ time per round.
\end{theorem} 

Our mistake bound matches (when viewed as an offline PAC learning result)
the best known error guarantees of the corresponding  offline learners of
Massart halfspaces with margin given in \cite{DGT19,CKMY20}. Moreover, the mistake-bound 
achieved by our algorithm is essentially best-possible when considering computationlly efficient (statistical query) algorithms: in the recent works
\cite{DK20-SQ-Massart,diakonikolas2022sq,NasserT22} that consider offline (PAC) learning
with Massart noise, it is shown that, even when the underlying distribution has
$\gamma$-margin, no polynomial-time algorithm can achieve classification error
better than $\eta/\polylog(1/(1-2\eta))$ in the Statistical Query framework. By
a standard online to offline reduction this readily implies a $\eta T$ lower
bound (up to $\polylog(1/(1-2 \eta))$ factors).  

\begin{remark}[Information-Theoretic vs Computationally Efficient Online Learning]
The $\eta T$ mistake bound \emph{is not information-theoretically optimal}:
in particular, in \cite{ben2009agnostic}) better, near-optimal, mistake bounds are given albeit with
inefficient algorithms (i.e., with runtime exponential in the dimension $d$).
When considering computationally efficient algorithms, as we do here, there is strong evidence 
(see the SQ lower bounds of \cite{DK20-SQ-Massart,diakonikolas2022sq}]) 
that the $\eta T$ mistake is essentially best possible.
Before our work no computationally efficient algorithm was known that could beat the random guessing benchmark (that makes $T/2$ mistakes).
Moreover, we remark that, even in the offline setting of linear classification under Massart noise, the first computationally efficient that achieved classification error $\eta$ (the offline equivalent of $\eta T$ mistakes) was only given in the relatively (given the history of the problem) recent work \cite{DGT19}.
\end{remark}

Our algorithm is particularly simple: we perform online gradient descent on a sequence of reweighted Leaky-ReLU loss functions.  The Leaky-ReLU loss has been successfully used in several works on learning with Random Classification and Massart noise \cite{Bylander:98a,DGT19,CKMY20}.  In the online setting however simply using (online) gradient descent on 
the Leaky-ReLU does not suffice: even though the adversary is restricted to
select examples with margin with respect to some ground-truth halfspace, they can still select examples very close to the decision boundary of the current hypothesis which would cause online gradient descent to get stuck or converge to sub-optimal solutions. To overcome this issue, we reweight the Leaky-ReLU loss by the margin of the current example according to the current hypothesis vector, see \Cref{alg:leaky-relu-1}.  We are then able to show that standard regret guarantees for Online Convex optimization can be translated to obtain mistake
bounds in the presence of Massart noise, see \Cref{lem:expected-regret}.

We next present our result on semi-random ``Massart'' $k$-arm setting presented in 
\Cref{def:monotone-rewards}.  
In addition to the $\Delta$ ``reward-margin'' assumption of \Cref{def:monotone-rewards},
similarly to our online classification result of \Cref{thm:2-arm},
we require a $\gamma$ ``geometric-margin'' assumption for the contexts with respect to some halfspace.
We give an efficient bandit algorithm, see \Cref{alg:leaky-relu-2}, that is able to accumulate
expected reward roughly $\Delta T$ more than playing a random action at every round.

\begin{theorem}[Monotone $k$-arm Contextual Bandits] \label{thm:bandit}
Consider the monotone reward online setting of \Cref{def:monotone-rewards}.
Moreover, for some unit vector $\vec w^\ast \in \R^d$, assume that 
for all $t$, it holds that for all $i$ $\|\vec X^{(t)}_i\|_2 \leq 1$ and 
for all $i \neq j$, $|\vec w^\ast \cdot \matr X^{(t)}_i - \vec w^\ast \cdot \matr X^{(t)}_j)| \geq \gamma$.
There exists a bandit algorithm that runs in $\poly(d)$ time per round and 
selects a sequence of arms
$\alpha^{(1)}, \ldots, \alpha^{(T)} \in [k]$ that obtain expected reward
 \begin{align*}
\E\left[\sum_{t=1}^T \vec r\tth(\alpha\tth) \right]
&\geq  \E\left[\sum_{t=1}^T\frac{1}{k}\sum_{i=1}^k \vec r_i\tth\right]  
+\frac{k-1}{k}\Delta T -  O(T^{5/6}(k \Delta M^2)^{1/3}/\gamma) \,.
 \end{align*}
\end{theorem}

Our algorithm for the $k$-arm bandit setting relies on the observation that, 
assuming that we could observe all rewards $r_i^{(t)}$, then one could treat
the labeled pairs $(\x_i^{(t)} - \x_j^{(t)}, r_i^{(t)} - r_j^{(t)})$ as real-valued
versions of online linear classification with Massart noise and provide it as 
input to our online learning algorithm.  As is common in bandit problems, to adapt 
this ``full-information'' approach to the bandit setting, we pick a random action with small
probability at every round that provides us with unbiased estimates of the full reward
vectors.  For more details we refer to \Cref{sec:bandits}. Finally, we remark that for the special case of 2-armed bandits our result implies the following corollary. 
\begin{corollary}
    [Monotone $2$-arm Contextual Bandits] \label{thm:bandit-2arm}
Consider the monotone reward online setting of \Cref{def:monotone-rewards}.
In the bandit setting, \Cref{alg:leaky-relu-2} produces a sequence of arm choices 
$\alpha^{(1)}, \ldots, \alpha^{(T)} \in \{1,2\}$ that obtain expected reward
 \begin{align*}
     \E\left[\sum_{t=1}^T \vec r\tth(\alpha\tth) \right]&\geq  \E\left[\sum_{t=1}^T
     \frac{\vec r_1^{(t)} + \vec r_2^{(t)}}{2} \right]  
+\frac{1}{2}\Delta T -  O(( M^2\Delta)^{1/3} T^{5/6}/\gamma) \,.
 \end{align*}
\end{corollary}
\Cref{thm:bandit-2arm} is a generalization of our online learning result of \Cref{thm:2-arm}:
indeed, using \Cref{rem:classification-to-bandit}, we obtain that the expected reward is equal to $T - M(T)$, where $M(T)$ are the expected mistakes. Thus, \Cref{thm:bandit-2arm}
implies that $T - M(T) \geq T/2 + (\Delta/2) T + o(T)$ and, using the fact that $\Delta = 1-2\eta$ and $M = 1$, we get that $M(T) \leq \eta T + o(T)$.  Given the hardness result for improving 
upon $\eta T$ regret for online Massart classification that we already discussed we conclude that our reduction from Bandit to Online Massart is tight for 2-armed bandits.

\subsection{Notation}
For $n \in \Z_+$, let $[n] \eqdef \{1, \ldots, n\}$.  We use lowercase boldface characters for vectors. We use $\bx \cdot \by $ for the inner product 
of $\bx, \by \in \R^d$.
For $\bx \in \R^d$, $\|\bx\|_2$ denotes the
$\ell_2$-norm of $\bx$.
For a set of vectors $k$, $\matr X=\{\x_1,\ldots,\x_k\}$, $\matr X_{i-j}$ denotes the difference between the vectors $\x_i,\x_j$, i.e., $\matr X_{i-j}=\x_i-\x_j$.
We use $\1_A = \1\{A\}$ to denote the
characteristic function of the set $A$.
We use the standard $O(\cdot), \Theta(\cdot), \Omega(\cdot)$ asymptotic notation.

We use $\E_{X\sim \D}[X]$ for the expectation of a random variable $X$ according to the
distribution $\D$ and $\pr[\mathcal{E}]$ for the probability of event $\mathcal{E}$. For simplicity of notation, we  omit the distribution when it is clear from the context.  
\section{Online Learning with Massart Noise}

In this section, we provide an algorithm for online learning halfspaces with Massart noise. The learner iteratively processes a sequence of covariates and associated labels $(\x,y)$ provided by the adversary, chooses a label corresponding to the current decision vector $\vec w$, and then updates (even if the label was correct)  the decision vector using a carefully chosen convex loss. Our loss is based on a modification of the LeakyReLU loss, i.e., $\mathrm{LeakyReLU}_\lambda(t) \triangleq (1-\lambda)t\1\{t>0\}+\lambda t\1\{t<0\}$. First, we show an equivalent expression for the Leaky-ReLU loss.
\begin{fact}
    Let $\mathrm{LeakyReLU}_\lambda(t) \triangleq (1-\lambda)t\1\{t>0\}+\lambda t\1\{t<0\}$. It can be equivalently expressed as $\mathrm{LeakyReLU}_\lambda(t) \triangleq 1/2((1-2\lambda)|t|-t)$.
\end{fact}
\begin{proof}
    Note that it holds $t\1\{t>0\}=(t+|t|)/2$ and $t\1\{t<0\}=(t-|t|)/2$. Therefore by plugging these identities in the definition of $\mathrm{LeakyReLU}_\lambda(t)$, we obtain the result.
\end{proof}
In our algorithm, we use as loss the $\mathrm{LeakyReLU}_\lambda(-ty)$ where $t\in \R$ and $y\in \{\pm 1\}$.  For the sake of brevity in notation, we define $C_{\Delta}(t;y)=\mathrm{LeakyReLU}_{(1-\Delta)/2}(-ty)$, i.e.,
\[
C_\Delta(t; y) \triangleq \frac12 (\Delta |t| - y t)  \,.
\]
We provide some intuition behind our choice of the Leaky-ReLU loss.
Notice, that in \Cref{def:online-massart} we have the label consistency which corresponds to the condition 
$\E[ y]  
\sgn(\vec w^\ast \cdot \x) \geq (1-2\eta)\eqdef\Delta$. This is clear, from the fact that $y\neq \sign(\wstar\cdot\x)$  with probability at most $\eta$ and that, 
$\E_y[y]=\E_y[\sign(\wstar\cdot\x)(\1\{y= \sign(\wstar\cdot\x )\}-\1\{y\neq \sign(\wstar\cdot\x )\})]=\sign(\wstar\cdot\x)(1-2\E[\1\{y\neq \sign(\wstar\cdot\x )\}])$,
therefore $\E[ y]  
\sgn(\vec w^\ast \cdot \x) \geq \Delta$. The Leaky-ReLU loss has the property that $\E_y[C_{\Delta}(-\wstar\cdot \vec x; y)]\leq 0$ (see \Cref{clm:optimal}), which we can later use to design a loss for our problem. However, optimizing this loss function exclusively does not guarantee minimal regret.
To this end, we define the loss function $\ell_{\vec u,\tau,\Delta}(\vec w,y,\x)\eqdef \frac{C_{\Delta} 
                  (\vec w \cdot \vec x ; y ) }{\max(|\vec u \cdot \vec x|,\tau)}$.
                  To simplify the notation, we define
                  $\ell\tth(\vec w)=\frac{C_{\Deltapar} 
                  (\vec w \cdot \vec x\tth ; y^{(t)} ) }{\max(|\vec w\tth \cdot \vec x\tth|,\tau)}$, where $\tau,\Deltapar$ are fixed and the other parameters are changing through the iterations of our algorithm. 
                  
Subsequently, we demonstrate that minimizing the regret associated with these reweighted chosen loss functions concurrently yields substantive guarantees for the regret in the context of \Cref{def:online-massart}.

\begin{algorithm}
    \centering
    \fbox{\parbox{5.1in}{
\begin{enumerate}

            \item $\Deltapar\gets 1-2\eta -\eps$ and $\tau\gets \eps\gamma/4$
              \item $\vec w^{(0)}=\vec e_1$.
              \item For $t=1,\ldots, T$:
              \begin{enumerate}[leftmargin=0.1cm]

                  \item Adversary selects point $\vec x^{(t)}\in \R^d$ and generates label $y^{(t)}$.
                  \item Learner observes $\vec x^{(t)}$ and chooses label $\hat{y}^{(t)}=
                  \sign( \vec w\tth \cdot \x^{(t)})$
                  \item Learner gets label $y^{(t)}$.
\item  Set 
                    
                  \[
                  \ell\tth(\vec w)=
                  \frac{C_{\Deltapar} 
                  (\vec w \cdot \vec x\tth ; y^{(t)} ) }{\max(|\vec w\tth \cdot \vec x\tth|,\tau)}
                  \]
                  \item Run Online Convex Optimization on $\ell\tth(\cdot)$.
              \end{enumerate}
            \end{enumerate}
    }}
        \caption{Online Learning Massart Halfspaces}
    \label{alg:leaky-relu-1}
\end{algorithm}

\begin{lemma}\label{lem:expected-regret}
Assume that a sequence $\vec w\tth$ is produced by a (possibly randomized) online 
algorithm $\mathcal{A}$  with the guarantee that 
\[
\E\left[\sum_{t=1}^T \ell\tth(\vec w\tth) -  \sum_{t=1}^T \ell\tth(\vec w^\ast)\right] 
\leq \ocoreg(T) \,,
\]
with $\ell\tth(\vec w) =
\frac{C_{\Deltapar}(\vec w\cdot \X; y^{(t)})}{\max(|\vec w^{(t)}\cdot \X|,\tau)}
$.
Then, it holds that 
\[
\sum_{t=1}\E[(\1\{\sign(\vec w\tth\cdot  \X)\neq y\tth\})]\leq T\eta + R(T,\eps,\gamma,\tau)\;,
\]
where $  R(T,\eps,\gamma,\tau)=T(\frac{\eps}{2} + \frac{  8\tau}{\eps \gamma}) +\ocoreg(T)(1+\frac{  8\tau}{\eps \gamma})$.
\end{lemma}
\begin{proof}
Denote $\eps=\Delta-\Deltapar$, with $\eps<\Delta/2$ and let $\tau\leq \eps\gamma/2$. First, we show that the optimal decision vector $\wstar$ gets negative loss on expectation.

\begin{claim}\label{clm:optimal}
It holds 
$\E_y[C_{\Deltapar}(\wstar\cdot \vec x; y)]\leq  -(\eps/2)|\wstar\cdot \x|$
\end{claim}
\begin{proof}
It holds that $C_{\Deltapar}(t;y)=\frac 12
(\Deltapar|t|-yt)$, therefore, we have that
\begin{align*}
    \E_y[C_{\Deltapar}(\wstar\cdot \x; y)]&=\frac 12 (\Deltapar|\wstar\cdot \x|- \E[y]\wstar\cdot \x)
    \\&=\frac 12 (\Deltapar- \E[y]\sign(\wstar\cdot \x))|\wstar\cdot \x|
    \\&\leq \frac 12  (\Deltapar  -\Delta)|\wstar \cdot \x|\;,
\end{align*}
where we used that $\E[y]\sign(\wstar\cdot \x)\geq \Delta$.
\end{proof}

We first show that 
\[
\E[\sum_{t=1}^T \ell\tth(\wstar))]
\leq - T \gamma \eps/2 \,.
\]
Using the tower rule, we have that
\[
\E[\sum_{t=1}^T \ell\tth(\wstar))]= \sum_{t=1}^T\E[ \ell\tth(\wstar))]\;.
\]
From \Cref{clm:optimal}, we have that 
$\E_y[C_{\Deltapar}(-\wstar\cdot \x\tth ;y)]\leq -(\eps/2) |\wstar\cdot \x\tth|$. Hence, we have that
\begin{align}
  \sum_{t=1}^T\E[ \ell\tth(\wstar))] &\leq -(\eps/2)\sum_{t=1}^T\frac{|\wstar\cdot \X|}{\max(|\vec w^{(t)}\cdot \X|,\tau)}\nonumber
 \\&\leq -(\eps\gamma/2)\sum_{t=1}^T\frac{1}{\max(|\vec w^{(t)}\cdot \X|,\tau)}\label[ineq]{eq:optimal-bound}\;,
\end{align}
where we used that $|\wstar\cdot \X|\geq \gamma$ from the assumptions.
Note that it holds that
$C_{\Deltapar}(\vec w\cdot \X;y)= (1/2)(\Deltapar-y\sign(\vec w\cdot \X))|\vec w\cdot \X|$.  Let $g\tth(y)=(1/2)(\Deltapar-y\sign(\vec w\tth \cdot \X ))$. Using the tower rule, we get 
\begin{align*}
&\E\left[\sum_{t=1}^T \ell\tth(\vec w\tth) 
-
\sum_{t=1}^T \ell\tth(\vec w^\ast) \right]
\\&=
\sum_{t=1}^T \left(g\tth(\E[y\tth])\frac{|\vec w^{(t)}\cdot \X|}{\max(|\vec w^{(t)}\cdot \X|,\tau)}- \E[\ell\tth(\vec w^\ast)]\right)
\;.
\end{align*}

We define $J$ to be the set of rounds where the $|\vec w^{(t)}\cdot \X|$ is smaller than the threshold  $\tau$, i.e.,
$J= \{t:|\vec w^{(t)}\cdot \X|\leq \tau\} $.
We have that
\begin{align}\label[ineq]{eq:enadio}
  \sum_{t\not\in J}g\tth(\E[y])\frac{|\vec w^{(t)}\cdot \X|}{\max(|\vec w^{(t)}\cdot \X|,\tau)}&=\sum_{t\not\in J}g\tth(\E[y])\geq -(T-|J|)/2\;,
\end{align}
and that 
\[
\sum_{t\in J}g\tth(\E[y])\frac{|\vec w^{(t)}\cdot \X|}{\max(|\vec w^{(t)}\cdot \X|,\tau)}\geq - |J|/2\;.
\]

Moreover note that from \Cref{eq:optimal-bound}, it holds that $
\sum_{t\in J} -\E[\ell\tth(\vec w^\ast)] \geq \frac{\eps\gamma}{2\tau} |J|$ and that  $
\sum_{t\not \in J} -\E[\ell\tth(\vec w^\ast)] \geq 0$.
Therefore, we get that 
\[
\sum_{t\in J}^T \E[\ell\tth(\vec w\tth) ]
-
\sum_{t\in J}^T \E[\ell\tth(\vec w^\ast)] \geq |J| \frac{\eps\gamma}{2\tau }-|J| \frac 12\;,
\]
and by using our assumption that $\tau\leq \eps\gamma/2$, we get that  $\sum_{t\in J}^T \E[\ell\tth(\vec w\tth) ]
-
\sum_{t\in J}^T \E[\ell\tth(\vec w^\ast)] \geq |J| \frac{\eps\gamma}{4\tau }$.
 Using the assumption for the regret guarantee, we have that
$
 \E[\sum_{t=1}^T\ell\tth(\vec w^{(t)})] \leq \ocoreg(T)+\E[\sum_{t=1}^T \ell\tth(\wstar)]\;,
$ which is equivalent to
 \begin{equation}
     \underbrace{\sum_{i\not\in J}E[\ell^{(t)}(\vec w^{(t)})]}_{I_1}\leq \bar{R}(T) \underbrace{-\sum_{i\in J}(E[\ell^{(t)}(\vec w^{(t)})]-E[\ell^{(t)}(\wstar)])}_{I_2}+\underbrace{\sum_{i\not\in J}E[\ell^{(t)}(\wstar)]}_{I_3}\;.\label[ineq]{eq:enatria}
 \end{equation}

Using that $I_1\geq (|J|-T)/2$ from \Cref{eq:enadio}, that $I_2\leq -|J|\frac{\eps\gamma}{4\tau}$ from above and that $I_3\leq 0$ from \Cref{clm:optimal}, we have that
\begin{align}
     |J|\leq (\ocoreg(T) +T)4\tau/(\eps \gamma)\;.\label[ineq]{eq:bound-stepsnew}
\end{align}
Finally, from the regret guarantees, i.e., \Cref{eq:enatria}, we have that 
$\sum_{t\not\in J}\E[\ell\tth(\vec w\tth)]\leq \ocoreg(T)$.
Recall that that $g\tth(\E[y]) = (1/2) ( \Deltapar -\E[y]\sign(\vec w\tth\cdot \X) )$. By adding $    \sum_{t\in J}(\Deltapar - \sign(\vec w\tth\cdot  \X)\E[y\tth])$ on both sides of \Cref{eq:enatria}  and using that $g\tth(\E[y]) \leq 2$ for all $t\in[T]$,  we have:
\begin{align}\label[ineq]{eq:bound}
    \sum_{t\in J}(\Deltapar - \sign(\vec w\tth\cdot  \X)\E[y\tth]) 
    +\sum_{t\not\in J}(\Deltapar - \sign(\vec w\tth\cdot  \X)\E[y\tth ])
    \leq 2 \ocoreg(T) + 2|J| \;,
\end{align}
where we have used that $  \sum_{t\in J}(\Deltapar - \sign(\vec w\tth\cdot  \X)\E[y\tth]) \leq 2|J|$ and that $I_2,I_3\leq 0$ as discuseed above.
Equivalently, using \Cref{eq:bound-stepsnew}, we obtain the following bound over all rounds $t=1,\ldots, T$:
$
    \sum_{t=1}(\Deltapar - \sign(\vec w\tth\cdot  \X)\E[y\tth ])
    \leq 2 \ocoreg(T) + (\ocoreg(T) +T){4\tau}/{(\eps \gamma)} \;.
$
Using that $\sign(\vec w\tth\cdot  \X)y\tth=1-2\1\{\sign(\vec w\tth\cdot  \X)\neq y\tth\}$ and $\Deltapar=1-2\eta -\eps$, we get that
\begin{align*}
    \sum_{t=1}\E[(\1\{\sign(\vec w\tth&\cdot  \X)\neq y\tth\})]
    \leq T(\eta+\eps)+\ocoreg(T) + 16(\ocoreg(T) +T)\tau/(\eps \gamma)\;.
\end{align*}
\end{proof}
Before proceeding with the proof of \Cref{thm:2-arm}, we make use of the following fact about the online gradient descent on convex functions to bound the $\ocoreg(T)$ of \Cref{lem:expected-regret}.
\begin{fact}[\cite{hazan2016introduction}]
\label{lem:ocoreg}
Let $\ell\tth$ be a convex function and let $D=\textrm{diam}(\mathcal W)$ and let $G=\max_t\|\nabla \ell\tth(\cdot)\|_2$. Online gradient descent with step size: $\eta={D}/({G\sqrt{T}})$ guarantees the following for all $T\geq 1$:
$
R(T) = 
\sum_{t=1}^T \ell\tth(\vec w^{(t)})- \min_{\vec w}\sum_{t=1}^T \ell\tth(\vec w)
\leq  {G D}/{3} \sqrt{T} 
\,.
$
\end{fact}

\begin{proof}[Proof of \Cref{thm:2-arm}]
Using \Cref{lem:ocoreg}, we get that $\ocoreg(T)=O(\sqrt{T}/\gamma)$, therefore from \Cref{lem:expected-regret}, we get that the expected total regret is bound above by
\begin{align*}
   \E[\sum_{t=1}(\1\{\sign(\vec w\tth&\cdot  \X)\neq y\tth\})]\leq \eta T + R(T,\eps,\gamma,\tau)\;,
\end{align*}
where $R(T,\eps,\gamma,\tau)=(\eps/2) T +\ocoreg(T) +  (\ocoreg(T) +T){ 8 \tau}/{(\eps \gamma)}$
To minimize this quantity, we set $\tau= \Theta(\eps^{1+\zeta}\gamma)$ for any $\zeta>0$ and $\eps=\Theta(T^{-1/(4+2\zeta)}/\gamma)$ gives  $R(T,\eps,\gamma,\tau)\leq T^{3/4-O(\zeta)}/\gamma$. By taking $\zeta$ close to $0$ and get the result.
\end{proof}

\section{Contextual Bandits with Monotone Rewards}
\label{sec:bandits}
 In this section, we describe an algorithm for the setting of \Cref{def:monotone-rewards}.
To this end, we use a generalization of the LeakyLeRU loss we used in \Cref{thm:2-arm}. This loss is described in \Cref{alg:generator}.  First we show that our loss satisfies some properties required by our main algorithm to work. The proof can be found at \Cref{app:sec4}.
\begin{algorithm}[h]
    \centering
    \fbox{\parbox{5.1in}{

\begin{enumerate}
     
                  \item Generate reward differences and context for every $j \in\{1,\ldots, k\}$: $                  \vec y_j =\vec r_{\alpha} - \vec r_j$.
                  \item If $\vec u=\vec 0$:
       Return the loss
                    
                  \[
                  \ell(\vec w) =
                  G(\vec w;\matr X, \vec v, \vec r, \alpha)
                  \triangleq - \Lambda\sum_{j \neq \alpha} \vec w\cdot \matr X_{\alpha-j} y_j
                  \]
                    
                  \item Otherwise:
                  \begin{enumerate}[leftmargin=0.1cm]

                      \item for every $j \in\{1,\ldots, k\}$ set:
                        
                  \[\vec z_j =
                  \matr X_{\alpha-j} + \rho \sgn(\matr X_{\alpha-j} 
                  \cdot \vec v) \frac{\vec v }{{\|\vec v\|_2}}
                  \]
                    
                  \item  Return the loss
                  
                  \[
                  \ell(\vec w) =
                  G(\vec w;\matr X, \vec v, \vec r, \alpha)
                  \triangleq
                  \sum_{j \neq \alpha} 
                  \frac{C_{\Delta} 
                  (\vec w \cdot \vec z_j; \vec y_j ) }
                  {|\vec v \cdot \vec z_j|}
                  \]
\end{enumerate}
    \end{enumerate}
}}
    \caption{Compute Loss $G(\vec w;\matr X, \vec v, \vec r, \alpha)$,
    where $\vec w$ is the argument and 
    $\matr X, \vec v, \vec r, \alpha$ are parameters.}
    \label{alg:generator}
\end{algorithm}
\begin{claim}\label{clm:G-convex-lipschitz} The loss $\ell(\vec w) =
                  G(\vec w;\matr X, \vec v, \vec r, \alpha)$ computed by \Cref{alg:generator} is convex
                   and $2 M k\max(\Lambda,1/\rho)$-Lipschitz.
\end{claim}
\begin{algorithm}
    \centering
    \fbox{\parbox{5.1in}{
\begin{enumerate}

                \item  $\vec w^{(0)}=\vec e_1$.
                             
             \item For $t=1,\ldots, T$:
            \begin{enumerate}[leftmargin=0.1cm]

                  \item 
                  Adversary picks context $\vec X^{(t)} = 
                  (\vec x_1^{(t)}, \ldots, \x_k^{(t)})$ 
                  and samples reward $\vec r\tth \sim D^{(t)}$.
                  \item If $\vec w\tth=\vec 0$: Learner picks a uniformly random $\alpha\tth$.
                    \item Otherwise:
   Learner picks 
                  $\alpha^{(t)} = 
                  \argmax_i \vec w\tth \cdot \x_i^{(t)}$.
                  \item Learner flips a coin $c\tth$ with HEADS probability $q$.
                  \item If HEADS:
                  \begin{enumerate}[leftmargin=0.1cm]
\item Learner picks uniformly random action $\beta^{(t)}$.
                  \item Learner gets reward $\vec r\tth (\beta\tth)$ and defines the fake reward vector
                  $\widetilde{\vec r}\tth$:
                  \begin{align*}
                  \widetilde{\vec r}\tth(i)= 
                  \begin{cases}  
                  (k-1) ~ \vec r \tth (\beta \tth) ~~ &\text{if} ~ i = \beta\tth
                  \\
                  M - \vec r \tth (\beta \tth) ~~ &\text{if} ~ i \neq \beta\tth \,,
                  \end{cases}
                  \end{align*}
\item  Set
                  $\ell\tth (\vec w) = 
                  \frac{1}{q}
                  \Gloss(\vec w; \matr X\tth, \vec w^{(t)}, 
                  \widetilde{\vec r}\tth, \alpha\tth)
                  $.
                  \end{enumerate}
                  \item If TAILS:
                  \begin{enumerate}[leftmargin=0.1cm]
\item 
                  Learner gets reward $\vec r\tth(\alpha\tth)$ 
                  and sets $\ell\tth(\vec w) = 0$.
                  \end{enumerate}
                  \item Learner performs 
                  Online Convex Optimization with loss $\ell\tth(\cdot)$. 
                  \end{enumerate}
                               
              \end{enumerate}
    }}
        \caption{Bandits with Monotone Rewards}
    \label{alg:leaky-relu-2}
\end{algorithm}

\begin{lemma}\label{lem:reward-bandits}
Let $\vec w\tth$ be the sequence produced by algorithm
\Cref{alg:leaky-relu-2} in the bandit setting of \Cref{def:monotone-rewards} 
with exploration probability $q$. It holds 
\begin{align*}
\E
[  \sum_{t=1}^T
( G(\vec w\tth; \matr X\tth, \vec r\tth, &\alpha\tth) 
-  G(\vec w^\ast; \matr X\tth, \vec r\tth, \alpha\tth) 
)
]
\\&\leq O(k M \sqrt{T}\Lambda/q )\;,
\end{align*}
where the expectation is over the randomness of \Cref{alg:leaky-relu-2}
and the randomness of the reward vectors $\vec r\tth \sim D\tth$.
\end{lemma}
\begin{proof}
We first show that for any reward vector $\vec r\tth$ the loss $\ell\tth(\vec w)$ that we construct 
at every round is an unbiased estimate of the corresponding full-information 
loss function $G(\vec w; \matr X\tth, \vec w^{(t)}, \vec r\tth, \alpha\tth)$.
We show the following claim.
\begin{claim}[Unbiased Loss Estimate]\label{clm:unbiased-loss}
For any $\vec w \in \R^d$ it holds 
\[
\E_{c\tth, \beta\tth} 
\left[ 
 \ell\tth(\vec w)  
 \mid \mathcal{F}^{(t)}, \vec r\tth
\right]
= G(\vec w; \vec X^{(t)}, \vec w\tth, \vec r^{(t)}, \alpha^{(t)} )
\,.
\]
\end{claim}
\begin{proof}
Since at every step of \Cref{alg:leaky-relu-2} we construct
the loss $G(\vec w; \vec X\tth, \vec w\tth, \wt{\vec r}\tth, \alpha\tth)$,
using \Cref{alg:generator}, we denote by $\wt{\vec y}\tth = \wt{\vec r}_{\alpha \tth}\tth - \wt{\vec r}_j\tth$
the reward difference vector of the adapted reward vector $\wt{\vec r}$.
Moreover, recall that we denote $\vec y\tth_j = \vec r\tth_{\alpha\tth} -
\vec r\tth_j$.
Notice that  $\wt{\vec y}\tth$  is a random variable that depends on the uniformly 
random action $\beta\tth$. The loss depends on whether the $\vec w\tth=\vec 0$ or not so, we consider each case separately. We start with the case where  $\vec w\tth\neq \vec 0$ and we denote this event as $A\tth$.
Recall that $\vec z\tth_j = \matr X_{\alpha\tth - j}\tth + 
\rho \sign(\matr X\tth_{\alpha\tth - j} \cdot  \vec w\tth) \vec w\tth/\|\vec w\tth\|_2$ is the vector containing the context differences as computed in \Cref{alg:generator}.
Taking the expectation with respect to the random coin flip $c\tth$ we obtain:
\begin{align*}
&\E_{c\tth, \beta\tth}  \left[  \ell\tth(\vec w)  \1\{A\tth\} \mid \mathcal{F}^{(t)}, \vec r\tth \right]
\\&=
q \frac{1}{q} 
\E_{\beta\tth}[G(\vec w; \vec X\tth, \vec w\tth, \wt{\vec r}\tth, \alpha\tth)\1\{A\tth\}  \mid \mathcal{F}\tth, \vec r\tth] \\
&= \E_{\beta\tth}\left[  \sum_{j \neq \alpha\tth} 
\frac{C_{\Delta}(\vec w\cdot \vec z\tth_j; \wt{\vec y}\tth_j)}
{|\vec w\tth \cdot \vec z\tth_j|}
\1\{A\tth\} ~ \bigg|~ \mathcal{F}\tth, \vec r\tth \right]
\\
&= 
\sum_{j \neq \alpha\tth} 
\frac{ \E_{\beta\tth}[ C_{\Delta}(\vec w\cdot \vec z\tth_j; \wt{\vec y}\tth_j) \mid  \mathcal{F}\tth, \vec r\tth ]}
{|\vec w\tth \cdot \vec z\tth_j|}\1\{A\tth\} 
\,.
\end{align*}
where for the last equation we used the linearity of expectation and 
the fact that the action $\alpha\tth$, the event $A\tth$ and the weight vector at $t$-th iteration 
$\vec w\tth$ do not depend on $\beta\tth$ conditional on $\mathcal{F}\tth$.
We now observe that the loss $C_{\Delta}(t; y) = (1/2) (\Delta |t| - y t)$ 
is linear in $y$.  Therefore, using again the linearity of expectation, we have that 
\begin{align*}
    \E_{\beta\tth}[ C_{\Delta}(\vec w\cdot \vec z\tth_j; &\wt{\vec y}\tth_j) \mid  \mathcal{F}\tth, \vec r\tth ]
\\&=
C_{\Delta}(\vec w\cdot \vec z\tth_j;  \E_{\beta\tth}[  \wt{\vec y}\tth_j \mid  \mathcal{F}\tth, \vec r\tth ]
)\,.
\end{align*}
Next, we consider the case where $\vec w\tth=\vec 0$, and we call this event $(A^c)\tth$. We have that by
taking the expectation with respect to the random coin flip $c\tth$ we obtain:
\begin{align*}
&\E_{c\tth, \beta\tth}  \left[  \ell\tth(\vec w)  \1\{(A^c)\tth\} \mid \mathcal{F}^{(t)}, \vec r\tth \right]
\\&= 
\E_{\beta\tth}[G(\vec w; \vec X\tth, \vec w\tth, \wt{\vec r}\tth, \alpha\tth)\1\{(A^c)\tth\}  \mid \mathcal{F}\tth, \vec r\tth] \\
&= \E_{\beta\tth}[ - \sum_{j \neq \alpha\tth} 
\vec w\cdot \matr X_{\alpha-j}  \wt{\vec y}\tth_j
\1\{(A^c)\tth\} ~ \bigg|~ \mathcal{F}\tth, \vec r\tth ]
\\
&= - \sum_{j \neq \alpha\tth} 
\vec w\cdot \matr X_{\alpha-j} \E_{\beta\tth}[  \wt{\vec y}\tth_j
 ~ \bigg|~ \mathcal{F}\tth, \vec r\tth ]\1\{(A^c)\tth\}\;,
\end{align*}
where for the last equation we used the linearity of expectation and 
the fact that the action $\alpha\tth$ and the event $A\tth$  do not depend on $\beta\tth$ conditional on $\mathcal{F}\tth$.

To finish the proof we have to show that the adapted reward difference vector $\wt{\vec y}\tth$ 
is an unbiased estimate of the true reward difference vector $\vec y\tth$.
Since we pick the action $\beta\tth$ uniformly at random from $\{1,\ldots, k\}$, we have
that the $i$-th coordinate of the adapted reward vector $\wt{\vec r}$ is equal to
\[
\E_{\beta\tth}[\wt{\vec r}\tth_i \mid  \mathcal{F}\tth, \vec r\tth ]
=
\frac{1}{k} (k-1) \vec r\tth_i
+ \frac{1}{k} \sum_{s \neq i}  (M - \vec r\tth_s) \,.
\]
Therefore, we have that the expected difference $\wt{\vec y}_j\tth$ is equal to 
\begin{align*}
&\E_{\beta\tth}[\wt{\vec y}_j\tth  \mid  \mathcal{F}\tth, \vec r\tth ]=
\frac{1}{k} (k-1) \vec r\tth_{\alpha\tth}
+ \frac{1}{k} \sum_{s \neq \alpha\tth}  (M - \vec r\tth_s) 
\\&-
\frac{1}{k} (k-1) \vec r\tth_{j}
- 
\frac{1}{k} \sum_{s \neq j}  (M - \vec r\tth_s) 
=
\frac{1}{k} (k-1) \vec r\tth_{\alpha\tth} 
- \frac{1}{k} \vec r\tth_{j}
\\&- \frac{1}{k} (k-1) \vec r\tth_{j} 
+ \frac{1}{k} \vec r\tth_{\alpha\tth}
=
\vec r\tth_{\alpha\tth} - \vec r\tth_{j}
=\vec y\tth_j
\,.
\end{align*}
Therefore, combining the above equations, we conclude that 
\begin{align*}
\E_{c\tth, \beta\tth}  &[  \ell\tth(\vec w)   \mid \mathcal{F}^{(t)}, \vec r\tth ]
G(\vec w; \matr X\tth, \vec w\tth, \vec r\tth, \alpha\tth) \,.
\end{align*}
This completes the proof of \Cref{clm:unbiased-loss}.
\end{proof}
We have that the loss $\ell\tth (\vec w) = G(\vec w; \matr X\tth, \vec w\tth, \wt{\vec r}\tth, \alpha\tth)$
constructed at each step $t$ of \Cref{alg:leaky-relu-2} is convex since it is 
the convex combination of the zero loss and the convex losses 
$G(\vec w;\matr X\tth, \vec w\tth, \vec r, \alpha\tth)$
(see \Cref{clm:G-convex-lipschitz}) for different reward vector vectors $\vec r$ 
(each realization of the random variable $\beta\tth$ corresponds to a different reward vector $\vec r$).
From \Cref{lem:ocoreg} and \Cref{clm:G-convex-lipschitz} we have that the sequence $\vec w\tth$ produced by
\Cref{alg:leaky-relu-2} achieves expected regret
\begin{align*}
\E
[  \sum_{t=1}^T
( G(\vec w\tth; \matr X\tth, \vec r\tth, \alpha\tth) 
- & G(\vec w^\ast; \matr X\tth, \vec r\tth, \alpha\tth) 
)
]
\\&\leq O(k M \sqrt{T}\Lambda/q )\;.\qedhere
\end{align*}
\end{proof}

Next, we prove a generalization of \Cref{lem:expected-regret} for the $k$-arm setting. It shows that minimizing the regret over the loss functions, bounds the expected reward of our setting.
\begin{lemma}\label{lem:convex-to-reward-reduction}
Let $\vec w\tth$ be a stochastic process in $\R^d$ adapted to the filtration $(\mathcal{F}\tth)_{t \in \mathcal{T}}$ such that 
\begin{align*}
&\E[\sum_{t=1}^T G(\vec w\tth;\matr X\tth,\vec w\tth,\vec r\tth,\alpha\tth)  \nonumber   - \sum_{t=1}^T G(\wstar;\matr X\tth,\vec w\tth,\vec r\tth,\alpha\tth)] 
\leq \ocoreg(T) \,.
\end{align*}

Then, it holds that 
\begin{align*}
\E[\sum_{t=1}^T \vec r_{\alpha\tth}\tth) ] \nonumber&\geq\E[\sum_{t=1}^T\frac{1}{k}\sum_{i=1}^k \vec r_i\tth]+(k-1)/k\Delta T - R(T,\gamma,\Delta,\Lambda,k) \,,
\end{align*}
with $R(T,\gamma,\Delta,\Lambda,k)=
(\ocoreg(T)/k)(1+1/(\gamma\Lambda))+T M/\Lambda$.
\end{lemma}
\begin{proof}
We denote as $A\tth$ the event that $\vec w\tth\neq \vec 0$ and $\rho=\gamma/2$.
We first observe that by adding $\rho \sign(\w \cdot \Xt{\alpha\tth-j}) \w$ to the
difference $\Xt{\alpha\tth-j}$ we do not affect the choice of the the 
optimal weight vector $\wstar$ and our guess 
$\w\tth$.
We observe that $\wstar\cdot \z_{j}\tth\sign(\wstar\cdot \Xt{\alpha\tth-j})
\geq(|\wstar\cdot \Xt{\alpha\tth-j}|-\rho \wstar\cdot \vec w\tth/\|\vec w\tth\|_2 )\geq \gamma-\rho\geq0$, therefore $\sign(\wstar\cdot \z_{j}\tth)=\sign(\wstar\cdot \Xt{\alpha\tth-j})$. For $\vec w\tth$ the similarly note that $ \vec w\cdot \z_{j}\tth=\sign(\vec w\tth\cdot \Xt{\alpha\tth-j})(|\vec w\tth\cdot \Xt{\alpha\tth-j}|+\rho\|\vec w\tth\|_2)$.

First, we show that the optimal decision vector $\wstar$ gets negative loss on expectation (see \Cref{app:sec4} for the proof).
\begin{claim}\label{clm:optimal-k-arms}
    It holds that \( \E_y[\sum_{t=1}^T \ell\tth(\wstar))] \leq  -k\gamma \Delta \sum_{t=1}^T \1\{(A\tth)^c\}\,. \)
\end{claim}
Next, we bound the contribution of the loss of the $\vec w\tth$ (See \Cref{app:sec4} for the proof).
\begin{claim}\label{clm:test1}
    It holds that 
    \begin{align*}
\E[\sum_{t=1}^T &\ell\tth(\vec w\tth) ]=
\sum_{t=1}^T\sum_{j\neq \alpha\tth} (1/2) ( \Delta -\sign(\vec w\tth\cdot  \Xt{\alpha\tth-j})\E[(\vec r_{\alpha\tth}\tth -\vec r_j\tth)] )\1\{A\tth\}
\;.
\end{align*}
\end{claim}
Let $J=\sum_{t=1}^T \1\{(A\tth)^c\}$, we show that can bound $J$ using our guarantees (See \Cref{app:sec4} for the proof). 
\begin{claim}\label{clm:bound-bad}
   It holds that $J\leq  (\ocoreg(T)+TM(k-1))/((k-1)\gamma\Delta \Lambda)$. 
\end{claim}
By plugging \Cref{eq:bandits11} and using \Cref{clm:optimal-k-arms} in the assumption for the regret guarantee, we get that
\begin{align}\label[ineq]{eq:bound-k}
 &\sum_{t=1}^T\sum_{j\neq \alpha\tth}\1\{A\tth\}(\Delta - \sign(\vec w\tth\cdot  \Xt{\alpha\tth-j})\E[\vec r_{\alpha\tth}\tth -\vec r_j\tth ])\leq 2 \ocoreg(T)\;.
\end{align}
Finally, we need to bound from below the term $\E\left[\sum_{t=1}^T \vec r_{\alpha\tth}\tth \right]$. For this reason, we need to connect the reward of each round $\vec r_{\alpha\tth}\tth$ with the regret of the loss \eqref{eq:bound-k}. Note that if the learner chose the action $\alpha\tth$ that means that $\vec w\tth\cdot  \Xt{\alpha\tth-j}\geq 0$ for all $j$ given that we are in the event $A\tth$. Therefore we can decompose $\vec r_{\alpha\tth}\tth$ as follows
\begin{align*}
 \vec r_{\alpha\tth}\tth &=\frac{1}{k}\sum_{i=1}^k \vec r_i\tth + \frac{1}{k}\sum_{j\neq \alpha\tth}^k\sign(\vec w\tth\cdot  \Xt{\alpha\tth-j})(\vec r_{\alpha\tth}\tth -\vec r_j\tth) \;.
\end{align*}
Furthermore, note that when we are in the event $(A^c)\tth$, the learner chooses a random action, therefore, using \Cref{eq:bound-k} and the equality above, we have that the expected reward is
\begin{align*}
\E[\sum_{t=1}^T  \vec r_{\alpha\tth}\tth  ]
&\geq\E[\sum_{t=1}^T\frac{1}{k}\sum_{i=1}^k \vec r_i\tth]   +\frac{k-1}{k}\Delta (T-J) -(2\ocoreg(T)/k) \;.    
\end{align*}
Using \Cref{clm:bound-bad} and setting $R(T,\gamma,\Delta,\Lambda,k)=
(\ocoreg(T)/k)(1+1/(\gamma\Lambda))+T M/\Lambda$, we complete the proof of \Cref{lem:convex-to-reward-reduction}.
\end{proof}

The proof of \Cref{thm:bandit} follows from \Cref{lem:convex-to-reward-reduction} and \Cref{lem:reward-bandits}. The proof can be found on \Cref{app:sec4}.

\section{Conclusions and Open Problems}
In this work we considered online linear classification in the Massart or Bounded online 
classification model of \cite{ben2009agnostic}.  Under a standard (and necessary) margin assumption, we gave the first efficient algorithm for 
this problem achieving a mistake bound of $\eta T + o(T)$.  This bound is essentially 
optimal due to known hardness results in the Statistical Query model \cite{diakonikolas2022sq}.  We extended our online learning setting to a $k$-arm Bandit model 
that lies between the commonly used regression-based, 
realizable  and the pessimistic agnostic classification contextual bandit models.  
In this model, we utilized our online Massart learner to obtain 
an efficient bandit algorithm that obtains roughly 
$(1-1/k) \Delta T$ more reward than playing at random at every round. 
We observed that our reduction is tight for the case of $2$ arms 
(given the aforementioned SQ hardness results for learning with Massart noise).  
However, for $k>2$ arms it is unclear whether this reward bound 
is best possible, as the gap between playing at random and playing 
the best arm at every round may be much larger than $\Delta$.  
We leave this as an interesting open problem for future work.
 \bibliographystyle{alpha}
\bibliography{clean2}
\appendix

\section*{Appendix}

\section{Omitted Proofs from \Cref{sec:bandits}}\label{app:sec4}
\subsection{Proof of \Cref{clm:G-convex-lipschitz}}
We prove the following:
\begin{claim} The loss $\ell(\vec w) =
                  G(\vec w;\matr X, \vec v, \vec r, \alpha)$ generated by \Cref{alg:generator} is convex
                   and $2 M k\max(\Lambda,1/\rho)$-Lipschitz.
\end{claim}
\begin{proof}
First, $\ell(\cdot)$ is convex because is a sum of convex functions. 
Moreover, note that the derivative is 
\[
\nabla \ell(\vec w)= \sum_{j\neq \alpha}\frac{1}{2}\frac{\Delta\sign(\vec w\cdot \vec z_j)-y }{|\vec v \cdot \vec z_j|}\vec z_j\;,
\]
and note that $|y|\leq \max_i |\vec r_i|$ and $\|\vec z_j\|_2\leq 2+\rho$ since all the $\x_i$ have norm at most $1$. Hence,
$\|\nabla \ell(\vec w)\|_2$ is upper bounded by $k\max_i |\vec r_i|\max_i(1/|\vec u\cdot \vec z_i|)\leq2 k \max_i |\vec r_i|/\rho$.
\end{proof}

\subsection{Proof of \Cref{clm:optimal-k-arms}}
We provide proof for the following claim:
\begin{claim}
    It holds that \( \E_y[\sum_{t=1}^T \ell\tth(\wstar))] \leq  -k\gamma \Delta \sum_{t=1}^T \1\{(A\tth)^c\}\,. \)
\end{claim}

\begin{proof}
Using the tower rule, we have that
\[
\E_y[\sum_{t=1}^T \ell\tth(\wstar))]= \sum_{t=1}^T\E_y[ \ell\tth(\wstar))\1\{A\tth\}+\1\{(A^c)\tth\}]\;.
\]
We first show that $\E_y[ \ell\tth(\wstar))\1\{A\tth\}]\leq 0$.
Recall that $C_{\Delta}(t;y)=\frac 12
(\Delta|t|-yt)$. By taking the expectation over $y$ in the $\E_y[C_{\Delta}(-\wstar\cdot \z_j\tth; y_j\tth)]$, we get that
\begin{align*}
    &\E_y[C_{\Delta}(-\wstar\cdot \z_j\tth; y_j\tth)] \\
    &=\frac 12 (\Delta|\wstar\cdot \z_j\tth|- \E\left[\left(\vec r_{\alpha\tth}\tth -\vec r_j\tth\right)\right]\wstar\cdot \z_j\tth) \;.
\end{align*}
Recall that by the \Cref{def:monotone-rewards}, we have that $\E[(\vec r_{\alpha\tth}\tth -\vec r_j\tth)]\sign(\wstar\cdot \Xt{\alpha\tth-j})\geq \Delta$ and as we discussed above $\sign(\wstar\cdot \z_j\tth)=\sign(\wstar\cdot \Xt{\alpha\tth-j})$. Therefore we have that 
\[
\E\left[\left(\vec r_{\alpha\tth}\tth -\vec r_j\tth\right)\right]\wstar\cdot \z_j\tth\geq \Delta |\wstar\cdot \z_j\tth|\;,
\]
which gives that $\E_y[C_{\Delta}(-\wstar\cdot \z_j\tth; y_j\tth)]\leq 0$ and hence $\E_y[ \ell\tth(\wstar))\1\{A\tth\}]\leq 0$ .
Next, with similar arguments as before we have that
$\E_y[ \ell\tth(\wstar))\1\{(A^c)\tth\}]\leq -k\Lambda \Delta \1\{(A^c)\tth\}$ which completes the proof.
\end{proof}
\subsection{Proof of \Cref{clm:test1}}
We prove the following:
\begin{claim}
    It holds that 
    \begin{align*}
\E[\sum_{t=1}^T &\ell\tth(\vec w\tth) ]=
\sum_{t=1}^T\sum_{j\neq \alpha\tth} (1/2) \big( \Delta \\&-\sign(\vec w\tth\cdot  \Xt{\alpha\tth-j})\E[(\vec r_{\alpha\tth}\tth -\vec r_j\tth)] \big)\1\{A\tth\}\
\;.
\end{align*}
\end{claim}
\begin{proof}
Note that in the case that the event $(A^c)\tth$ the loss of $\vec w\tth$ is zero.
Recall that it holds that
$C_{\Delta}(\vec w\tth\cdot \z_j\tth;y_j\tth)= (1/2)(\Delta-y_j\tth\sign(\vec w\tth\cdot \z_j\tth))|\vec w\tth\cdot \z_j\tth|$.  To simplify the notation, let $g_j\tth(y_j\tth)=(1/2)(\Delta-y_j\tth\sign(\vec w\tth \cdot \vec z_j\tth ))$. Using the tower rule, we get 
\begin{align}
\E&[\sum_{t=1}^T \ell\tth(\vec w\tth) ]\nonumber
\\&=
\sum_{t=1}^T\sum_{j\neq \alpha\tth} \frac{g_j\tth(\E[y_j\tth])|\vec w^{(t)}\cdot \z_j\tth|}{|\vec w^{(t)}\cdot \z_j\tth|}\1\{A\tth\}\nonumber
\\&=
\sum_{t=1}^T\sum_{j\neq \alpha\tth} g_j\tth(\E[y_j\tth])\1\{A\tth\}\label[ineq]{eq:bandits11}
\;.
\end{align}
Moreover, since it holds that $\sign(\vec w\tth\cdot \z_j\tth)=\sign(\vec w\tth\cdot \Xt{\alpha\tth-j})$, we have  $g_j\tth(\E[y_j\tth]) = (1/2) ( \Delta -\sign(\vec w\tth\cdot \Xt{\alpha\tth-j}) \E[y_j\tth]).$
Hence, we have that
\begin{align*}
    g_j\tth&(\E[y_j\tth]) \\&= (1/2) \big( \Delta -\sign(\vec w\tth\cdot  \Xt{\alpha\tth-j})\E[(\vec r_{\alpha\tth}\tth -\vec r_j\tth)] \big)\;.
\end{align*}
\end{proof}
\subsection{Proof of \Cref{clm:bound-bad}}
We restate and prove the following claim.
\begin{claim}
   It holds that $J\leq  (\ocoreg(T)+TM(k-1))/((k-1)\gamma\Delta \Lambda)$. 
\end{claim}
\begin{proof}
    From \Cref{eq:bandits11}, we have that 
 \begin{align*}
     \E\left[\sum_{t=1}^T \ell\tth(\vec w\tth) \right]\nonumber
&=
\sum_{t=1}^T\sum_{j\neq \alpha\tth} g_j\tth(\E[y_j\tth])\1\{A\tth\}
\\&\geq  -M(k-1)(T-J)/2\;.
 \end{align*}
 Furthermore, sing the assumption for the regret guarantee, we have:
\[
\E\left[\sum_{t=1}^T
 \ell\tth(\vec w^{(t)})- \sum_{t=1}^T \ell\tth(\wstar)\right]\leq \ocoreg(T)\;.
\]
Hence, using \Cref{clm:optimal-k-arms}, we get the result.
\end{proof}

\subsection{Proof of \Cref{thm:bandit}}
We restate and prove \Cref{thm:bandit}.
\begin{theorem}[Monotone $k$-arm Contextual Bandits]
Consider the monotone reward online setting of \Cref{def:monotone-rewards}.
Moreover, for some unit vector $\vec w^\ast \in \R^d$, assume that 
for all $t$, it holds that for all $i$ $\|\vec X^{(t)}_i\|_2 \leq 1$ and 
for all $i \neq j$, $|\vec w^\ast \cdot \matr X^{(t)}_i - \vec w^\ast \cdot \matr X^{(t)}_j)| \geq \gamma$.
There exists a bandit algorithm that runs in $\poly(d)$ time per round and 
selects a sequence of arms
$\alpha^{(1)}, \ldots, \alpha^{(T)} \in [k]$ that obtain expected reward
 \begin{align*}
\E\left[\sum_{t=1}^T \vec r\tth(\alpha\tth) \right]
&\geq  \E\left[\sum_{t=1}^T\frac{1}{k}\sum_{i=1}^k \vec r_i\tth\right]  
\\&+\frac{k-1}{k}\Delta T -  O(T^{5/6}(k \Delta M^2)^{1/3}/\gamma) \,.
 \end{align*}
\end{theorem}
\begin{proof}[Proof of \Cref{thm:bandit}]
\Cref{alg:leaky-relu-2} in each iteration, either with probability $q$ makes a random choice or with probability $1-q$ chooses the best action according to the current decision weight vector (or a random action if $\vec w\tth=\vec 0$). Therefore we have that
\begin{align*}
\sum_{t=1}^T \E_{c\tth}\left[ \vec r\tth  \right]= (1-q)\sum_{t=1}^T \E\left[ \vec r_{\alpha\tth}\tth  \right] +\frac{q}{k}\E\left[\sum_{t=1}^T\sum_{i=1}^k \vec r_i\tth\right] \;.
\end{align*}
Using \Cref{lem:convex-to-reward-reduction}, we get that
\begin{align*}
    \sum_{t=1}^T \E\left[ \vec r\tth  \right]&\geq \frac{1}{k}\E\left[\sum_{t=1}^T\sum_{i=1}^k \vec r_i\tth\right]
    \\&+(1-q)\frac{k-1}{k}\Delta T -(1-q)R(T,\gamma,\Delta,\Lambda,k) \;.
\end{align*}
where $R(T,\gamma,\Delta,\Lambda,k)=
(\ocoreg(T)/k)(1+1/(\gamma\Lambda))+T M/\Lambda$.
From \Cref{lem:reward-bandits} we have that 
$\ocoreg(T)=(k-1)M\sqrt{T}/q\max(1/\gamma,\Lambda)$. By maximizing it, we get $\Lambda=1/\gamma T^{1/6}(M/(k\Delta))^{1/3}$ and $q= M/(\gamma\Lambda\Delta)$. 
Therefore we get that the expected reward is
\begin{align*}
    \sum_{t=1}^T \E\left[ \vec r\tth  \right]&\geq \frac{1}{k}\E\left[\sum_{t=1}^T\sum_{i=1}^k \vec r_i\tth\right]
    \\&+\frac{k-1}{k}\Delta T -T^{5/6}(k \Delta)^{1/3}M^{2/3}/\gamma \;.
\end{align*}
\end{proof}
 \end{document}